\documentclass[twoside,11pt]{article}
\usepackage[preprint]{myjmlr2e}
\pdfoutput=1

\ShortHeadings{Deep Networks Are Kernel Machines}{Domingos}
\firstpageno{1}

\begin{document}

\thispagestyle{empty}

\title{Every Model Learned by Gradient Descent Is Approximately a Kernel Machine}

\author{\name Pedro Domingos \email pedrod@cs.washington.edu \\
\addr Paul G. Allen School of Computer Science \& Engineering \\
University of Washington \\
Seattle, WA 98195-2350, USA }

\editor{}

\maketitle

\begin{abstract}
Deep learning's successes are often attributed to its ability to automatically discover new representations of the data, rather than relying on handcrafted features like other learning methods. We show, however, that deep networks learned by the standard gradient descent algorithm are in fact mathematically approximately equivalent to kernel machines, a learning method that simply memorizes the data and uses it directly for prediction via a similarity function (the kernel). This greatly enhances the interpretability of deep network weights, by elucidating that they are effectively a superposition of the training examples. The network architecture incorporates knowledge of the target function into the kernel. This improved understanding should lead to better learning algorithms.
\end{abstract}


\section{Introduction}

Despite its many successes, deep learning remains poorly understood \citep{goodfellow16}. In contrast, kernel machines are based on a well-developed mathematical theory, but their empirical performance generally lags behind that of deep networks \citep{scholkopf02}. The standard algorithm for learning deep networks, and many other models, is gradient descent \citep{rumelhart86}. Here we show that every model learned by this method, regardless of architecture, is approximately equivalent to a kernel machine with a particular type of kernel. This kernel measures the similarity of the model at two data points in the neighborhood of the path taken by the model parameters during learning. Kernel machines store a subset of the training data points and match them to the query using the kernel. Deep network weights can thus be seen as a superposition of the training data points in the kernel's feature space, enabling their efficient storage and matching. This contrasts with the standard view of deep learning as a method for discovering representations from data, with the attendant lack of interpretability \citep{bengio13}. Our result also has significant implications for boosting algorithms \citep{freund97}, probabilistic graphical models \citep{koller09}, and convex optimization \citep{boyd04}.

\section{Path Kernels}

A kernel machine is a model of the form
\[ y = g \left(\sum_i a_i K(x,x_i) + b\right), \]
where $x$ is the query data point, the sum is over training data points $x_i$, $g$ is an optional nonlinearity, the $a_i$'s and $b$ are learned parameters, and the kernel $K$ measures the similarity of its arguments \citep{scholkopf02}. In supervised learning, $a_i$ is typically a linear function of $y_i^*$, the known output for $x_i$. Kernels may be predefined or learned \citep{cortes09}. Kernel machines, also known as support vector machines, are one of the most developed and widely used machine learning methods. In the last decade, however, they have been eclipsed by deep networks, also known as neural networks and multilayer perceptrons, which are composed of multiple layers of nonlinear functions. Kernel machines can be viewed as neural networks with one hidden layer, with the kernel as the nonlinearity. For example, a Gaussian kernel machine is a radial basis function network \citep{poggio90}. But a deep network would seem to be irreducible to a kernel machine, since it can represent some functions exponentially more compactly than a shallow one \citep{delalleau11,cohen16}. 

Whether a representable function is actually learned, however, depends on the learning algorithm. Most deep networks, and indeed most machine learning models, are trained using variants of gradient descent \citep{rumelhart86}. Given an initial parameter vector $w_0$ and a loss function $L = \sum_i L(y_i^*,y_i)$, gradient descent repeatedly modifies the model's parameters $w$ by subtracting the loss's gradient from them, scaled by the learning rate $\epsilon$:
\[ w_{s+1} = w_{s} - \epsilon \nabla_w L(w_s). \]
The process terminates when the gradient is zero and the loss is therefore at an optimum (or saddle point). Remarkably, we have found that learning by gradient descent is a strong enough constraint that the end result is guaranteed to be approximately a kernel machine, regardless of the number of layers or other architectural features of the model.

Specifically, the kernel machines that result from gradient descent use what we term a {\em path kernel}. If we take the learning rate to be infinitesimally small, the path kernel between two data points is simply the integral of the dot product of the model's gradients at the two points over the path taken by the parameters during gradient descent:
\[ K(x,x') = \int_{c(t)} \nabla_w y(x) \cdot \nabla_w y(x') \: dt, \]
where $c(t)$ is the path. Intuitively, the path kernel measures how similarly the model at the two data points varies during learning. The more similar the variation for $x$ and $x'$, the higher the weight of $x'$ in predicting $y$. Fig.~\ref{pkf} illustrates this graphically.

\begin{figure}[p]
\begin{center}
\includegraphics*[width=5in]{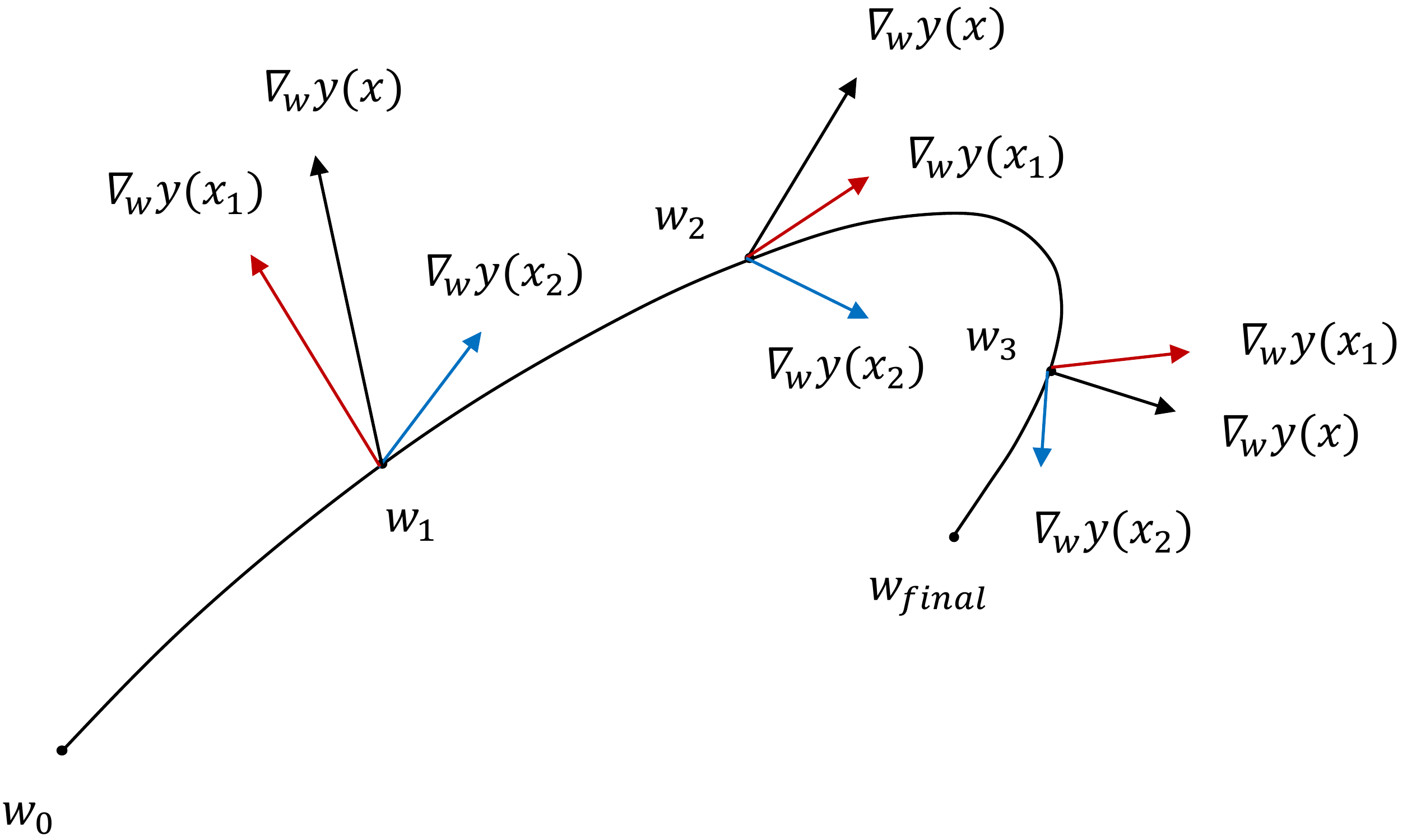}
\vspace{0.4in}
\caption{How the path kernel measures similarity between examples. In this two-dimensional illustration, as the weights follow a path on the plane during training, the model's gradients (vectors on the weight plane) for $x$, $x_1$ and $x_2$ vary along it. The kernel $K(x,x_1)$ is the integral of the dot product of the gradients $\nabla_w y(x)$ and $\nabla_w y(x_1)$ over the path, and similarly for $K(x,x_2)$. Because on average over the weight path $\nabla_w y(x) \cdot \nabla_w y(x_1)$ is greater than $\nabla_w y(x) \cdot \nabla_w y(x_2)$, $y_1$ has more influence than $y_2$ in predicting $y$, all else being equal.
\label{pkf}}
\end{center}
\end{figure}

Our result builds on the concept of neural tangent kernel, recently introduced to analyze the behavior of deep networks \citep{jacot18}. The neural tangent kernel is the integrand of the path kernel when the model is a multilayer perceptron. Because of this, and since a sum of positive definite kernels is also a positive definite kernel \citep{scholkopf02}, the known conditions for positive definiteness of neural tangent kernels extend to path kernels \citep{jacot18}. A positive definite kernel is equivalent to a dot product in a derived feature space, which greatly simplifies its analysis \citep{scholkopf02}.

We now present our main result. For simplicity, in the derivations below we assume that $y$ is a (real-valued) scalar, but it can be made a vector with only minor changes. The data points $x_i$ can be arbitrary structures.

\begin{definition}
\hfill The {\em tangent kernel} associated with function $f_w(x)$ and parameter vector $v$ is \\ $K^g_{f,v}(x,x') = \nabla_w f_w(x) \cdot \nabla_w f_w(x')$, with the gradients taken at $v$.
\label{gkd}
\end{definition}

\begin{definition}
The {\em path kernel} associated with function $f_w(x)$ and curve $c(t)$ in parameter space is $K^p_{f,c}(x,x') = {\displaystyle \int_{c(t)}} K^g_{f,w(t)}(x,x') \: dt$.
\label{pkd}
\end{definition}

\begin{theorem}
Suppose the model $y = f_w(x)$, with $f$ a differentiable function of $w$, is learned from a training set $\{ (x_i, y_i^*) \}_{i=1}^m$ by gradient descent with differentiable loss function $L = \sum_i L(y_i^*,y_i)$ and learning rate $\epsilon$. Then
\[ \lim_{\epsilon \rightarrow 0} y = \sum_{i=1}^m a_i K(x,x_i) + b, \]
where $K(x,x_i)$ is the path kernel associated with $f_w(x)$ and the path taken by the parameters during gradient descent, $a_i$ is the average $- \partial L / \partial y_i$ along the path weighted by the corresponding tangent kernel, and $b$ is the initial model.
\label{pkt}
\end{theorem}

\begin{proof}
In the $\epsilon \rightarrow 0$ limit, the gradient descent equation, which can also be written as
\[ \frac{w_{s+1} - w_s}{\epsilon} = - \nabla_w L(w_s), \]
where $L$ is the loss function, becomes the differential equation
\[ \frac{dw(t)}{dt} = - \nabla_w L(w(t)). \]
(This is known as a gradient flow \citep{ambrosio08}.) Then for any differentiable function of the weights $y$,
\[ \frac{dy}{dt} = \sum_{j=1}^d \frac{\partial y}{\partial w_j} \frac{dw_j}{dt}, \]
where $d$ is the number of parameters. Replacing $dw_j/dt$ by its gradient descent expression:
\[ \frac{dy}{dt} = \sum_{j=1}^d \frac{\partial y}{\partial w_j} \left(- \frac{\partial L}{\partial w_j} \right). \]
Applying the additivity of the loss and the chain rule of differentiation:
\[ \frac{dy}{dt} = \sum_{j=1}^d \frac{\partial y}{\partial w_j} \left(- \sum_{i=1}^m \frac{\partial L}{\partial y_i} \frac{\partial y_i}{\partial w_j} \right). \]
Rearranging terms:
\[ \frac{dy}{dt} = - \sum_{i=1}^m \frac{\partial L}{\partial y_i} \sum_{j=1}^d \frac{\partial y}{\partial w_j} \frac{\partial y_i}{\partial w_j}. \]
Let $L'(y_i^*,y_i) = \partial L / \partial y_i$, the loss derivative for the $i$th output. Applying this and Definition~\ref{gkd}:
\[ \frac{dy}{dt} = - \sum_{i=1}^m L'(y_i^*,y_i) K^g_{f,w(t)}(x,x_i). \]
Let $y_0$ be the initial model, prior to gradient descent. Then for the final model $y$:
\[ \lim_{\epsilon \rightarrow 0} y = y_0 - \int_{c(t)} \sum_{i=1}^m L'(y_i^*,y_i) K^g_{f,w(t)}(x,x_i) \: dt, \]
where $c(t)$ is the path taken by the parameters during gradient descent. Multiplying and dividing by $\int_{c(t)} K^g_{f,w(t)}(x,x_i) \: dt$:
\[ \lim_{\epsilon \rightarrow 0} y =  y_0 - \sum_{i=1}^m \left( \frac{\int_{c(t)} K^g_{f,w(t)}(x,x_i) L'(y_i^*,y_i) \: dt}{\int_{c(t)} K^g_{f,w(t)}(x,x_i) \: dt} \right) \int_{c(t)} K^g_{f,w(t)}(x,x_i) \: dt. \]
Let $\overline{L'}(y_i^*,y_i) = \int_{c(t)} K^g_{f,w(t)}(x,x_i) L'(y_i^*,y_i) \: dt / \int_{c(t)} K^g_{f,w(t)}(x,x_i) \: dt$, the average loss derivative weighted by similarity to $x$. Applying this and Definition~\ref{pkd}:
\[ \lim_{\epsilon \rightarrow 0} y = y_0 - \sum_{i=1}^m \overline{L'}(y_i^*,y_i) K^p_{f,c}(x,x_i). \]
Thus
\[ \lim_{\epsilon \rightarrow 0} y = \sum_{i=1}^m a_i K(x,x_i) + b, \]
with $K(x,x_i) = K^p_{f,c}(x,x_i)$, $a_i = - \overline{L'}(y_i^*,y_i)$, and $b = y_0$.
\end{proof}

\begin{remark}
This differs from typical kernel machines in that the $a_i$'s and $b$ depend on $x$. Nevertheless, the $a_i$'s play a role similar to the example weights in ordinary SVMs and the perceptron algorithm: examples that the loss is more sensitive to during learning have a higher weight. $b$ is simply the prior model, and the final model is thus the sum of the prior model and the model learned by gradient descent, with the query point entering the latter only through kernels. Since Theorem~\ref{pkt} applies to every $y_i$ as a query throughout gradient descent, the training data points also enter the model only through kernels (initial model aside).
\end{remark}

\begin{remark}
Theorem~\ref{pkt} can equally well be proved using the loss-weighted path kernel $K^{lp}_{f,c,L} = \int_{c(t)} L'(y_i^*,y_i) K^g_{f,w(t)}(x,x_i) \: dt$, in which case $a_i = -1$ for all $i$.
\end{remark}

\begin{remark}
In least-squares regression, $L'(y_i^*,y_i) = y_i - y_i^*$. When learning a classifier by minimizing cross-entropy, the standard practice in deep learning, the function to be estimated is the conditional probability of the class, $p_i$, the loss is $- \sum_{i=1}^m \ln p_i$, and the loss derivative for the $i$th output is $- 1 / p_i$. Similar expressions hold for modeling a joint distribution by minimizing negative log likelihood, with $p_i$ as the probability of the data point.
\end{remark}

\begin{remark}
\hfill Adding \hfill a \hfill regularization \hfill term \hfill $R(w)$ \hfill to \hfill the \hfill loss \hfill function \hfill simply \hfill adds \\ $- \int_{c(t)} \sum_{j=1}^d (\partial y / \partial w_j) (\partial R / \partial w_j)$ to $b$.
\end{remark}

\begin{remark}
The proof above is for batch gradient descent, which uses all training data points at each step. To extend it to stochastic gradient descent, which uses a subsample, it suffices to multiply each term in the summation over data points by an indicator function $I_i(t)$ that is 1 if the $i$th data point is included in the subsample at time $t$ and 0 otherwise. The only change this causes in the result is that the path kernel and average loss derivative for a data point are now stochastic integrals. Based on previous results \citep{scieur17}, Theorem~\ref{pkt} or a similar result seems likely to also apply to further variants of gradient descent, but proving this remains an open problem.

\end{remark}

For linear models, the path kernel reduces to the dot product of the data points. It is well known that a single-layer perceptron is a kernel machine, with the dot product as the kernel \citep{aizerman64}. Our result can be viewed as a generalization of this to multilayer perceptrons and other models. It is also related to Lippmann {\em et al.}'s proof that Hopfield networks, a predecessor of many current deep architectures, are equivalent to the nearest-neighbor algorithm, a predecessor of kernel machines, with Hamming distance as the comparison function \citep{lippmann87}.

The result assumes that the learning rate is sufficiently small for the trajectory of the weights during gradient descent to be well approximated by a smooth curve. This is standard in the analysis of gradient descent, and is also generally a good approximation in practice, since the learning rate has to be quite small in order to avoid divergence (e.g., $\epsilon = 10^{-3}$) \citep{goodfellow16}. Nevertheless, it remains an open question to what extent models learned by gradient descent can still be approximated by kernel machines outside of this regime.

\section{Discussion}

A notable disadvantage of deep networks is their lack of interpretability \citep{zhang18}. Knowing that they are effectively path kernel machines greatly ameliorates this. In particular, the weights of a deep network have a straightforward interpretation as a superposition of the training examples in gradient space, where each example is represented by the corresponding gradient of the model. Fig.~\ref{wsf} illustrates this. One well-studied approach to interpreting the output of deep networks involves looking for training instances that are close to the query in Euclidean or some other simple space \citep{ribeiro16}. Path kernels tell us what the exact space for these comparisons should be, and how it relates to the model's predictions.

\begin{figure}[p]
\begin{center}
\includegraphics[width=5in]{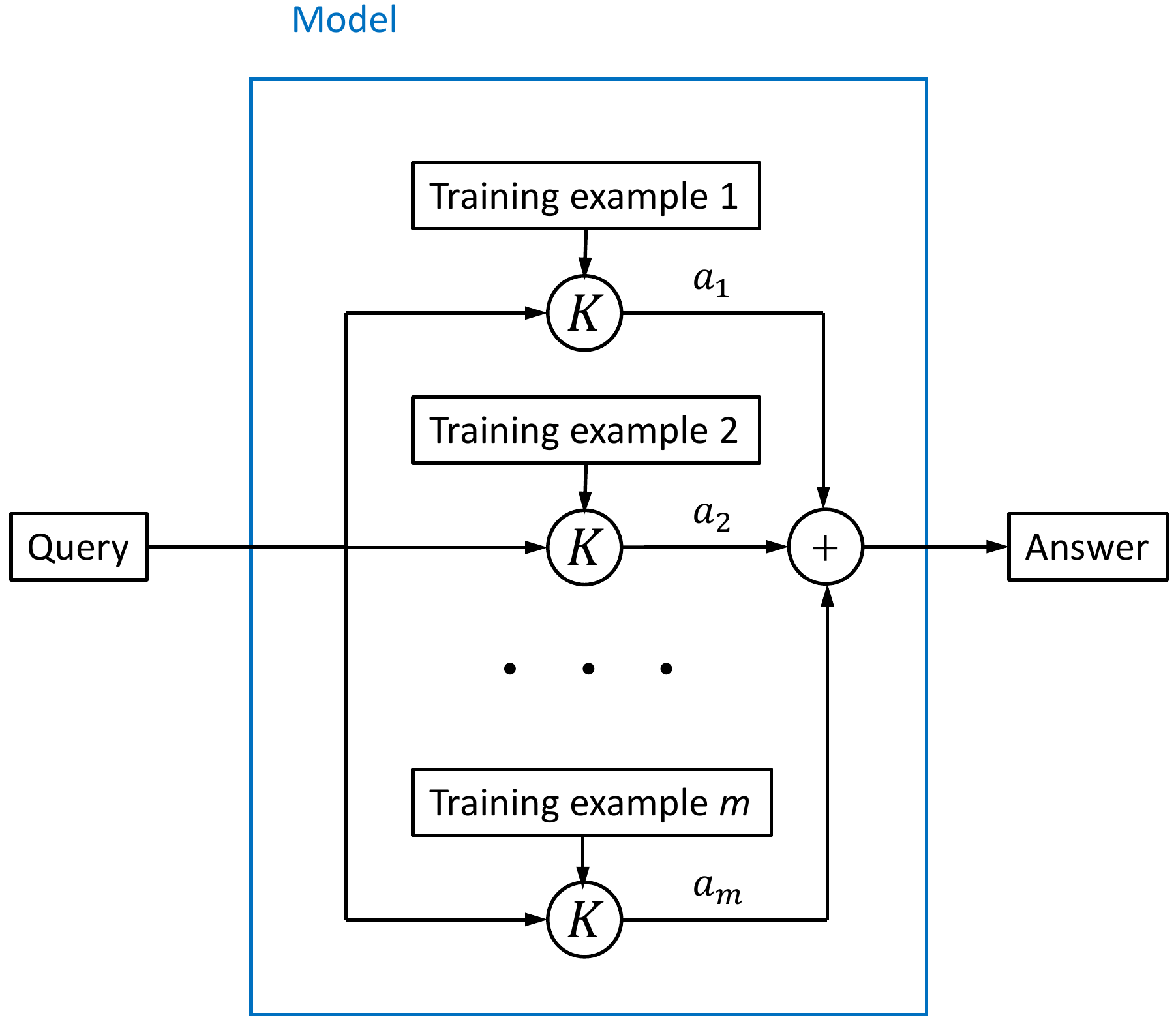}
\vspace{0.4in}
\caption{Deep network weights as superpositions of training examples. Applying the learned model to a query example is equivalent to simultaneously matching the query with each stored example using the path kernel and outputting a weighted sum of the results.
\label{wsf}}
\end{center}
\end{figure}

Experimentally, deep networks and kernel machines often perform more similarly than would be expected based on their mathematical formulation \citep{brendel19}. Even when they generalize well, deep networks often appear to memorize and replay whole training instances \citep{zhang17,devlin15}. The fact that deep networks are in fact kernel machines helps explain both of these observations. It also sheds light on the surprising brittleness of deep models, whose performance can degrade rapidly as the query point moves away from the nearest training instance \citep{szegedy14}, since this is what is expected of kernel estimators in high-dimensional spaces \citep{hardle04}.

Perhaps the most significant implication of our result for deep learning is that it casts doubt on the common view that it works by automatically discovering new representations of the data, in contrast with other machine learning methods, which rely on predefined features \citep{bengio13}. As it turns out, deep learning also relies on such features, namely the gradients of a predefined function, and uses them for prediction via dot products in feature space, like other kernel machines. All that gradient descent does is select features from this space for use in the kernel. If gradient descent is limited in its ability to learn representations, better methods for this purpose are a key research direction. Current nonlinear alternatives include predicate invention \citep{muggleton88} and latent variable discovery in graphical models \citep{elidan00}. Techniques like structure mapping \citep{gentner83}, crossover \citep{holland75} and predictive coding \citep{rao99} may also be relevant. Ultimately, however, we may need entirely new approaches to solve this crucial but extremely difficult problem.

Our result also has significant consequences on the kernel machine side. Path kernels provide a new and very flexible way to incorporate knowledge of the target function into the kernel. Previously, it was only possible to do so in a weak sense, via generic notions of what makes two data points similar. The extensive knowledge that has been encoded into deep architectures by applied researchers, and is crucial to the success of deep learning, can now be ported directly to kernel machines. For example, kernels with translation invariance or selective attention are directly obtainable from the architecture of, respectively, convolutional neural networks \citep{lecun98} or transformers \citep{vaswani17}.

A key property of path kernels is that they combat the curse of dimensionality by incorporating derivatives into the kernel: two data points are similar if the candidate function's derivatives at them are similar, rather than if they are close in the input space. This can greatly improve kernel machines' ability to approximate highly variable functions \citep{bengio05}. It also means that points that are far in Euclidean space can be close in gradient space, potentially improving the ability to model complex functions. (For example, the maxima of a sine wave are all close in gradient space, even though they can be arbitrarily far apart in the input space.)

Most significantly, however, learning path kernel machines via gradient descent largely overcomes the scalability bottlenecks that have long limited the applicability of kernel methods to large data sets. Computing and storing the Gram matrix at learning time, with its quadratic cost in the number of examples, is no longer required. (The Gram matrix is the matrix of applications of the kernel to all pairs of training examples.) Separately storing and matching (a subset of) the training examples at query time is also no longer necessary, since they are effectively all stored and matched simultaneously via their superposition in the model parameters. The storage space and matching time are independent of the number of examples. (Interestingly, superposition has been hypothesized to play a key role in combatting the combinatorial explosion in visual cognition \citep{arathorn02}, and is also essential to the efficiency of quantum computing \citep{nielsen00} and radio communication \citep{carlson09}.) Further, the same specialized hardware that has given deep learning a decisive edge in scaling up to large data \citep{raina09} can now be used for kernel machines as well.

The significance of our result extends beyond deep networks and kernel machines. In its light, gradient descent can be viewed as a boosting algorithm, with tangent kernel machines as the weak learner and path kernel machines as the strong learner obtained by boosting it \citep{freund97}. In each round of boosting, the examples are weighted by the corresponding loss derivatives. It is easily seen that each round (gradient descent step) decreases the loss, as required.  The weight of the model at a given round is the learning rate for that step, which can be constant or the result of a line search \citep{boyd04}. In the latter case gradient descent is similar to gradient boosting \citep{mason99}.

Another consequence of our result is that every probabilistic model learned by gradient descent, including Bayesian networks \citep{koller09}, is a form of kernel density estimation \citep{parzen62}. The result also implies that the solution of every convex learning problem is a kernel machine, irrespective of the optimization method used, since, being unique, it is necessarily the solution obtained by gradient descent. It is an open question whether the result can be extended to nonconvex models learned by non-gradient-based techniques, including constrained \citep{bertsekas82} and combinatorial optimization \citep{papadimitriou82}.

The results in this paper suggest a number of research directions. For example, viewing gradient descent as a method for learning path kernel machines may provide new paths for improving it. Conversely, gradient descent is not necessarily the only way to form superpositions of examples that are useful for prediction. The key question is how to optimize the tradeoff between accurately capturing the target function and minimizing the computational cost of storing and matching the examples in the superposition.

\acks{This research was partly funded by ONR grant N00014-18-1-2826. Thanks to L\'eon Bottou and Simon Du for feedback on a draft of this paper.}

\vskip 0.2in
\bibliography{refs}

\end{document}